%% file: main.tex
\documentclass[11pt]{article}
\usepackage[margin=1in]{geometry}
\usepackage{amsmath,amssymb,amsthm,mathtools}
\usepackage{enumitem}
\usepackage[hidelinks,pdftitle={Thompson Sampling for Non-Monotone Convex Ridge Bandits: Monotonicity Is Not Needed for Polynomial Regret},pdfauthor={Xuan Li},pdfsubject={Bandit convex optimisation; Thompson sampling},pdfkeywords={bandit convex optimisation, Thompson sampling, ridge functions, information ratio}]{hyperref}
\usepackage{microtype}

\newtheorem{theorem}{Theorem}[section]
\newtheorem{lemma}[theorem]{Lemma}
\newtheorem{proposition}[theorem]{Proposition}

\newtheorem{definition}[theorem]{Definition}
\theoremstyle{remark}

\newcommand{\R}{\mathbb{R}}
\newcommand{\E}{\mathbb{E}}
\newcommand{\PP}{\mathbb{P}}
\newcommand{\Fblr}{\mathcal{F}_{\mathrm{blr}}}
\newcommand{\Fblrm}{\mathcal{F}_{\mathrm{blrm}}}
\newcommand{\Fbl}{\mathcal{F}_{\mathrm{bl}}}
\newcommand{\IR}{\mathrm{IR}}
\newcommand{\TS}{\mathrm{TS}}
\newcommand{\BReg}{\mathrm{BReg}}
\newcommand{\PAIR}{\mathrm{PAIR}}
\newcommand{\conv}{\mathrm{conv}}
\newcommand{\diam}{\mathrm{diam}}
\newcommand{\rank}{\mathrm{rank}}
\newcommand{\tr}{\mathrm{tr}}
\newcommand{\argmin}{\mathrm{arg\,min}}
\newcommand{\ip}[2]{\langle #1, #2\rangle}
\newcommand{\norm}[1]{\lVert #1\rVert}
\newcommand{\fbar}{\bar f}

\title{Thompson Sampling for Non-Monotone Convex Ridge Bandits:\\ Monotonicity Is Not Needed for Polynomial Regret}
\author{Xuan Li\\[2pt] \normalsize University of New South Wales, Sydney, Australia \textperiodcentered\ \texttt{winny.li@unsw.edu.au}\\ \href{https://orcid.org/0009-0002-0213-6991}{ORCID: 0009-0002-0213-6991}}
\date{\today}

\begin{document}
\maketitle

\begin{abstract}
Bakhtiari, Lattimore and Szepesv\'ari (COLT 2025) proved that Thompson sampling (TS) has Bayesian regret $\tilde O(d^{5/2}\sqrt n)$ for bandit convex optimisation with convex \emph{monotone} ridge losses $f(x)=\ell(\ip{x}{\theta})$, and asked whether monotonicity of the link is necessary. We give a qualitative negative answer. For every prior on $[0,1]$-valued, $1$-Lipschitz convex ridge losses with an arbitrary convex, possibly non-monotone, link, and for any fixed measurable selection of minimisers, exact-posterior TS has Bayesian regret $O\big((d+1)^4\sqrt{dn}\,\log(e+nd\max\{1,\diam K\})\big)=\tilde O(d^{9/2}\sqrt n)$. The monotone proof relies on a single-removal John-ellipsoid dichotomy; we show by an explicit twelve-point configuration that this dichotomy fails for non-monotone links, and replace it by an $O(d^2)$ cardinality bound for ``uninformative'' configurations. The bound uses a Boolean rounding argument: a $0$-$1$ matrix within $1/(4r)$ in max-norm of a rank-$r$ matrix has rank at most $2r-1$. We construct $d(d+1)$ uninformative losses, showing that the cardinality bound is tight up to constants in the large-diameter-to-gap regime, and give a self-contained information-ratio-to-regret transfer that is uniform over fixed measurable selections. Whether the $d^{5/2}$ dependence of the monotone case can be retained remains open.
\end{abstract}

\section{Introduction}

Bayesian bandit convex optimisation is the following game. A convex body $K\subset\R^d$ and a class $\mathcal F$ of convex functions $K\to[0,1]$ are known, together with a prior $\xi$ on $\mathcal F$. The environment samples $f\sim\xi$ once. In each round $t=1,\dots,n$ the learner plays $X_t\in K$ and observes $Y_t\in\{0,1\}$ with $\E[Y_t\mid X_1,Y_1,\dots,X_t,f]=f(X_t)$. The Bayesian regret of a learner $\mathcal A$ is
\[
\BReg_n(\mathcal A,\xi)=\E\Big[\sup_{x\in K}\sum_{t=1}^n\big(f(X_t)-f(x)\big)\Big].
\]
Thompson sampling (TS) samples $f_t$ from the posterior in every round and plays a minimiser $X_t=x_{f_t}$ of the sampled function. Bakhtiari, Lattimore and Szepesv\'ari~\cite{BLS25} analysed TS in this setting via the information ratio. Among their results, they showed that if $\xi$ is supported on \emph{monotone} convex ridge functions, i.e.\ $f(x)=\ell(\ip{x}{\theta})$ with $\ell:\R\to\R$ convex and non-decreasing, then $\BReg_n(\TS,\xi)=O(d^{2.5}\sqrt n\log^2(nd\,\diam K))$; and they showed that TS can fail catastrophically on general convex losses. In their discussion they write:
\begin{quote}
``At present we are uncertain whether or not the monotonicity assumption is needed in the ridge setting. Our best guess is that it is not.''
\end{quote}
The monotone ridge class is a Bayesian version of the generalised linear bandit with an unknown convex, increasing link. Dropping monotonicity allows links such as $\ell(s)=|s-s_0|$ whose minimising set on $K$ is, when $d\ge2$ and $s_0$ lies in the interior of the projection interval $\{\ip{x}{\theta}:x\in K\}$, a whole hyperplane section; minimisers of convex ridge losses can thus be non-unique, so the tie-breaking rule of TS matters, and the geometry that drives the monotone analysis (an ordering of minimisers along the ridge direction) is lost.

Our goal is statistical rather than computational: we ask whether the canonical exact-posterior TS rule itself is safe on this structured class. The existence of low-regret algorithms for convex ridge links does not settle this question. Lattimore~\cite{Lattimore21} obtains $O(d\sqrt n\log(nD))$ minimax regret for an adversarial ridge model with a different algorithm, but the question here is algorithm-specific: \cite{BLS25} show that standard TS can fail catastrophically on general high-dimensional convex losses, so structure is genuinely needed to certify TS itself, and the monotone ridge class was the one for which they could do so. Our result identifies convex ridge structure as sufficient even when the link is non-monotone: a qualitative robustness statement about the canonical Bayesian sampling rule, not a minimax-optimality claim.

\paragraph{Contributions.}
Let $\Fblr$ be the class of convex ridge functions $K\to[0,1]$ that are $1$-Lipschitz, with an arbitrary convex link (Section~\ref{sec:setting}).
\begin{enumerate}[leftmargin=2em]
\item \textbf{Regret bound without monotonicity (Theorem~\ref{thm:main}).} For every prior on $\Fblr$ and every fixed measurable selection of minimisers, TS satisfies $\BReg_n(\TS,\xi)=\tilde O(d^{9/2}\sqrt n)$; explicitly, $\BReg_n(\TS,\xi)\le 7+73{,}728\sqrt3\,(d+1)^4 d^{1/2}\sqrt n\,\log(e+nd\max\{1,\diam K\})$. This resolves the qualitative question of whether monotonicity is necessary for polynomial-in-$d$ Bayesian regret of exact-posterior TS; the quantitative question of matching the $d^{5/2}$ dependence of the monotone case remains open, and is discussed in Section~\ref{sec:discussion}.
\item \textbf{A cardinality bound for uninformative configurations (Theorem~\ref{thm:size}).} The information-ratio machinery of~\cite{BLS25} reduces to showing that a finite set of loss functions whose minimisers reveal nothing about one another is small. In the monotone case this is proved by a John-ellipsoid dichotomy. We show that this single-removal dichotomy has no non-monotone analogue (Proposition~\ref{prop:john}) and prove directly that every such configuration has at most $3(d+1)^2-(d+1)-1$ elements. The key tool is an elementary rounding lemma (Lemma~\ref{lem:round}): if a $0$-$1$ matrix is entrywise within $1/(4r)$ of a matrix of rank $r$ then its rank is at most $2r-1$. The bound is tight up to a constant: Proposition~\ref{prop:lower} exhibits uninformative configurations of size $d(d+1)$.
\item \textbf{Transfer from information ratio to regret under any fixed measurable selection rule (Theorem~\ref{thm:transfer}).} The covering argument of~\cite{BLS25} uses that ties among minimisers are broken consistently. Since minimisers of convex ridge losses can be non-unique, we give a self-contained transfer theorem for exact TS, valid for every fixed measurable selection rule, using a cover by infimal-convolution approximations and an information-ratio bound that is uniform over selection rules. The rule is arbitrary but fixed in advance; history-dependent tie-breaking is not covered.
\end{enumerate}

\paragraph{Related work.}
Ridge (single-index) bandits have been studied under different learning criteria and for different algorithms. Lattimore~\cite{Lattimore21} proves $O(d\sqrt n\log(nD))$ minimax regret for an adversarial ridge model with a hidden fixed direction and convex links that may vary over time, by information-theoretic arguments and minimax duality; as noted above, this establishes learnability of the ridge structure without bearing on standard TS under an arbitrary prior. Bakhtiari et al.~\cite{BLS25} analyse exact TS directly and prove the monotone ridge bound; the present paper removes monotonicity within the ridge structure. Rajaraman, Han, Jiao and Ramchandran~\cite{RHJR24}, Rajaraman and Han~\cite{RH26a} and Kang et al.~\cite{Kang26} study frequentist nonlinear ridge, single-index and contextual single-index bandits with algorithms other than TS, and Rajaraman and Han~\cite{RH26b} give minimax lower bounds for general stochastic bandit convex optimisation, which quantify the difficulty of general convex losses but are not specific to non-monotone ridge losses or to TS. None of these results provides a Bayesian regret guarantee for exact-posterior TS over arbitrary convex non-monotone ridge links. We do not address the \emph{second} question of the same paragraph of~\cite{BLS25} (the information ratio with a \emph{known} link); see also the MSc thesis of Bakhtiari~\cite{BakhtiariThesis}. Information-ratio analyses of TS go back to~\cite{RVR16}; the convex-bandit machinery we build on is from~\cite{BDKP15,BE18,Lattimore20,BLS25}.

\section{Setting and notation}\label{sec:setting}

Throughout, $K\subset\R^d$ is a convex body (compact, convex, non-empty interior) with $0\in K$, and $D=\diam(K)$. A function $f:K\to\R$ is a \emph{convex ridge function} if $f(x)=\ell(\ip{x}{\theta})$ for some convex $\ell:\R\to\R$ and $\theta\in\R^d$ ($\theta=0$ is allowed). We write
\[
\Fbl=\{f:K\to[0,1]\ \text{convex},\ \mathrm{Lip}(f)\le1\},\qquad
\Fblr=\{f\in\Fbl:\ f\text{ is a convex ridge function}\},
\]
and $\Fblrm\subset\Fblr$ for the subclass with non-decreasing link. The classes $\Fbl$ and $\Fblrm$ are those of~\cite{BLS25}; $\Fblr$ is not treated there.

\paragraph{Measurability and tie-breaking (standing convention).}
We equip $\Fblr$ with a $\sigma$-algebra $\mathcal A$ for which $f\mapsto f(x)$ is measurable for every $x\in K$ and a measurable selection $f\mapsto x_f\in\argmin_K f$ has been fixed once and for all (such selections exist; see~\cite[Appendix~B]{BLS25}). Pointwise measurability makes the identity $(\Fblr,\mathcal A)\to(C(K),\mathcal B)$ measurable, where $\mathcal B$ is the Borel $\sigma$-algebra of the uniform norm; $\mathcal A$ may be strictly finer than $\mathcal B$. All approximating functions constructed in Section~\ref{sec:transfer} are measurable as $(C(K),\mathcal B)$-valued maps, and there the information-ratio hypothesis is only applied to finitely-valued random elements of $\Fblr$, which are $\mathcal A$-measurable automatically; this is what allows arbitrary priors and arbitrary fixed selections. All statements hold for \emph{every} such selection rule, but the rule is fixed in advance: it may not depend on the history of the game. We write $f^\star=\min_K f=f(x_f)$. For a probability measure $\xi$ on $\Fblr$ we write $\fbar=\E_\xi[f]$ (pointwise), a convex function $K\to[0,1]$ which in general is \emph{not} a finite convex combination of elements of $\Fblr$ (for instance the average of $x\mapsto\max\{\ip{x}{\theta},0\}$ over uniformly random directions in the plane is a multiple of the Euclidean norm); the arguments below only use pointwise values of $\fbar$.

\paragraph{Thompson sampling.}
TS with prior $\xi$ is Algorithm~1 of~\cite{BLS25}: in round $t$ sample $f_t$ from the posterior $\PP(f\in\cdot\mid X_1,Y_1,\dots,X_{t-1},Y_{t-1})$ and play $X_t=x_{f_t}$. The observation model is Bernoulli: $Y_t\in\{0,1\}$ with $\E[Y_t\mid X_1,Y_1,\dots,X_t,f]=f(X_t)$.

\paragraph{Regret and information terms.}
Unless otherwise stated, a random function $h:K\to[0,1]$ is a Borel random element of $C(K)$ (so $(\omega,x)\mapsto h_\omega(x)$ is jointly measurable and $h^\star=\min_Kh$ is measurable); laws on $\Fbl$, and laws on $(\Fblr,\mathcal A)$ through the identity map, are included. For the law $\nu$ of such a random function and a policy $\pi\in\mathcal P(K)$, let $(X,h)\sim\pi\otimes\nu$, $\bar h=\E_\nu[h]$, and
\[
\Delta(\pi,\nu)=\E[\bar h(X)-h^\star],\qquad I(\pi,\nu)=\E\big[(h(X)-\bar h(X))^2\big].
\]
For $\xi\in\mathcal P(\Fblr)$ let $\pi^\xi_{\TS}$ be the law of $x_f$ under $f\sim\xi$. Following~\cite{BLS25},
\[
\IR(\mathcal F)=\Big\{(\alpha,\beta)\in\R_+^2:\ \sup_{\xi\in\mathcal P(\mathcal F)}\big[\Delta(\pi^\xi_{\TS},\xi)-\alpha-\sqrt{\beta I(\pi^\xi_{\TS},\xi)}\big]\le 0\Big\}.
\]
Since $\pi^\xi_{\TS}$ depends on the selection, so does $\IR(\mathcal F)$; we say $(\alpha,\beta)\in\IR(\mathcal F)$ \emph{uniformly} if it holds for every measurable selection.

\paragraph{The decomposition lemma.}
The following is a restatement, for the prior mean $\fbar=\E_\xi f$, of the decomposition device of~\cite[Lemma~3]{BLS25}, which is stated there for $\fbar\in\conv(\mathcal F)$. Since $\E_\xi f$ need not be a finite convex combination, we include a proof in Appendix~\ref{app:a}.
\begin{lemma}[Decomposition; after {\cite[Lemma~3]{BLS25}}]\label{lem:decomp}
Let $\mathcal F\subseteq\Fblr$ and $\alpha,\beta_0\ge0$. Suppose there are integers $k\ge2$ and $m\ge1$ such that for every prior $\xi\in\mathcal P(\mathcal F)$, with $\fbar=\E_\xi f$, there is a measurable partition $\mathcal F=\bigcup_{i=1}^m\mathcal F_i$ such that, with the maximum taken over the pieces of positive $\xi$-measure,
\[
\max_{i\in[m]}\Big[\sup_{f\in\mathcal F_i}(\fbar(x_f)-f^\star)-\alpha-\Big(\beta_0\inf_{f_1,\dots,f_k\in\mathcal F_i}\sum_{(j,l)\in\PAIR(k)}\big(f_j(x_{f_l})-\fbar(x_{f_l})\big)^2\Big)^{1/2}\Big]\le 0,
\]
where $\PAIR(k)$ is the set of ordered pairs of distinct indices. Then $(\alpha,\beta_0k(k-1)m)\in\IR(\mathcal F)$ for the selection used to define $x_f$.
\end{lemma}
The proof averages the hypothesis over $k$ independent draws from each piece and applies Cauchy--Schwarz over the $m$ pieces, using $I(\pi^\xi_{\TS},\xi)\ge\sum_i\xi(\mathcal F_i)^2J_i$ for the within-piece informations $J_i$; only pointwise values of $\fbar$ are used.

\paragraph{Notation.}
$h:=d+1$. For a unit vector $\theta$ and a finite set $C$ of functions with selected minimisers, $t_{fg}=\ip{x_g}{\theta_f}$. $\PAIR(C)$ is the set of ordered pairs of distinct elements of $C$. Logarithms are natural.

\section{Main results}\label{sec:results}

\begin{theorem}[Regret of TS on non-monotone convex ridge losses]\label{thm:main}
Let $d\ge1$, let $K\subset\R^d$ be a convex body with $0\in K$ and diameter $D$, let $\xi$ be any prior on $\Fblr$ (with respect to any $\sigma$-algebra under which point evaluations and the selection are measurable, as in Section~\ref{sec:setting}), let $f\mapsto x_f$ be any such fixed measurable selection of minimisers, and let observations be Bernoulli as above. Then for every $n\ge1$,
\[
\BReg_n(\TS,\xi)\ \le\ 7+73{,}728\sqrt3\,(d+1)^4\,d^{1/2}\,\sqrt n\,\log\big(e+nd\max\{1,D\}\big).
\]
\end{theorem}

Theorem~\ref{thm:main} combines an information-ratio bound (Theorem~\ref{thm:ir}) with a transfer theorem (Theorem~\ref{thm:transfer}); the former reduces, through the blocking argument of~\cite{BLS25}, to a cardinality bound for the following objects.

\begin{definition}[Uninformative configurations]\label{def:uninf}
Fix $c\ge12$, $\varepsilon>0$ and set $\delta=\varepsilon/(c(d+1))$. A \emph{configuration} is a finite set $C\subset\Fblr$ with pairwise distinct selected minimisers together with a convex function $\fbar:K\to[0,1]$ (the \emph{comparison function}) such that for all $f,g\in C$,
\[
r_f:=\fbar(x_f)-f^\star\in[\varepsilon/2,\varepsilon],\qquad |\fbar(x_f)-\fbar(x_g)|\le\delta.
\]
A row $f\in C$ is \emph{uninformative} if $|f(x_g)-\fbar(x_g)|<\delta$ for all $g\in C\setminus\{f\}$; the configuration is \emph{uninformative} if every row is. For fixed $c$ and $d$ we write $g_c(d,z)$ for the supremum of $|C|$ over all uninformative configurations, over all convex bodies $K\ni0$ in $\R^d$, all $\varepsilon>0$ with $\diam(K)/\varepsilon\le z$, all comparison functions and all selections.
\end{definition}

\begin{theorem}[Cardinality bound]\label{thm:size}
Let $c=96(d+1)$. Every uninformative configuration in $\Fblr$ satisfies $|C|\le 3(d+1)^2-(d+1)-1<4(d+1)^2$; that is, $g_c(d,z)\le3(d+1)^2-(d+1)-1$ for all $z>0$.
\end{theorem}

\begin{theorem}[Information ratio]\label{thm:ir}
Let $h=d+1$, $k=3072\,h^4$ and, for $\alpha\in(0,1)$, $m_\alpha=1+\lceil\log_2(1/\alpha)\rceil$. Then
\[
\big(\alpha,\ 12\,k(k-1)\,m_\alpha\big)\in\IR(\Fblr)\quad\text{uniformly over measurable selections;}
\]
in particular $\IR(\Fblr)$ contains a pair $(\alpha,\beta)$ with $\beta\le 113{,}246{,}208\,(d+1)^8\,m_\alpha$.
\end{theorem}

\begin{proposition}[Tightness of the cardinality bound]\label{prop:lower}
For every $d\ge2$ and every $c\ge12$ there is an uninformative configuration in $\Fblr$, on a truncated simplex of diameter at most $\sqrt2$ in normalised coordinates, with $|C|=d(d+1)$ and $D/\varepsilon\le 64\sqrt2\,c^2d(d+1)^2$; rescaling the functions about the common comparison value shows that the same size is attained for every larger value of $D/\varepsilon$. In particular, for $c=96(d+1)$ and $z\ge589{,}824\sqrt2\,d(d+1)^4$, $d(d+1)\le g_c(d,z)\le3(d+1)^2-(d+1)-1$. For $d=1$, $g_c(1,z)\le2$ for every $z>0$ and $g_c(1,z)=2$ for $z\ge4$.
\end{proposition}

\begin{proposition}[The single-removal John-ellipsoid dichotomy has no non-monotone analogue]\label{prop:john}
For $d=3$ and every $c\ge12$ there is an uninformative configuration $C$ of twelve functions in $\Fblr$ such that, writing $J_\zeta(C)=\conv\{x_f:f\in C\}+B_\zeta$ and $E_\zeta(\cdot)$ for the maximum-volume inscribed ellipsoid, $E_\zeta(C\setminus\{f\})=E_\zeta(C)$ for every $f\in C$ and every $\zeta>0$.
\end{proposition}

Proposition~\ref{prop:john} concerns the proof strategy of~\cite{BLS25}, not their results: their Lemma~8 (``some removal shrinks the John ellipsoid by a factor $0.85$, or some ordered pair is informative'') is stated and proved for the monotone class only, and Proposition~\ref{prop:john} shows that its literal analogue for $\Fblr$ is false for every fixed shrink factor $\gamma<1$ in place of $0.85$. It does not exclude other uses of John ellipsoids.

\subsection{Proof overview}\label{sec:overview}

The argument has five steps; we record for each its input, its output and its cost in the dimension, so that the exponent $9/2$ can be traced.
\begin{enumerate}[leftmargin=2em,itemsep=1pt]
\item \emph{Two thin bands (Section~\ref{sec:structure}).} Every $f\in\Fblr$ has a representation $\ell(\ip{\cdot}{\theta})$ with $\ell$ convex, $1$-Lipschitz and globally minimised at the projections of the minimisers (Lemma~\ref{lem:rep}). If a row $f$ of a configuration is uninformative, convexity of $\ell$ along $\theta_f$ forces the projections of all other minimisers into two thin bands at the extreme projections, on either side of $\ip{x_f}{\theta_f}$, each of relative width $\eta\approx6/(c(d+1))$ (Lemma~\ref{lem:bands}). For monotone links only one band can be non-empty; the second band is the whole difficulty.
\item \emph{Balanced and unbalanced rows (Section~\ref{sec:size}).} Rows whose two band distances $L,R$ are comparable are handled by a quadratic evaluation matrix of rank at most $h(h+1)/2$ with small off-diagonal entries, giving at most $h^2-1$ such rows (Lemma~\ref{lem:tracerank}). For the remaining rows this quadratic estimate deteriorates when the two band distances are highly imbalanced, so we eliminate the far-band entries with an exactly low-rank Boolean mask instead: the near-band membership matrix $H$, a $0$-$1$ matrix, is within $1/(8h)$ of an affine evaluation matrix of rank at most $h$, and Lemma~\ref{lem:round} gives $\rank H\le2h-1$. Masking an affine matrix by $H$ (Lemma~\ref{lem:hadamard}) then yields a near-identity matrix of rank at most $h(2h-1)$, hence at most $2h^2-h$ unbalanced rows. Altogether $|C|\le3h^2-h-1=O(d^2)$ (Theorem~\ref{thm:size}), which is tight up to constants (Proposition~\ref{prop:lower}). The rounding tolerance $1/(4r)$ with $r=h$ is what forces $c=\Theta(d)$.
\item \emph{Counting (Section~\ref{sec:ir}).} If a configuration has $G+q$ elements, deleting first endpoints of informative pairs $q$ times shows that its ordered pairs carry energy at least $q\delta^2$ (Lemma~\ref{lem:count}). Within one dyadic level of the gap $\fbar(x_f)-f^\star$, a $k$-tuple is cut into $b=4ch$ blocks of $s=2G$ functions; at least half the blocks are configurations, each contributing $G\delta^2$, so every $k$-tuple has energy at least $\varepsilon^2/12$ with $k=bs=3072h^4$. Lemma~\ref{lem:decomp} turns this into $(\alpha,12k(k-1)m_\alpha)\in\IR(\Fblr)$, i.e.\ $\beta=O(d^8\log(1/\alpha))$ (Theorem~\ref{thm:ir}). The exponent $8$ is $2\times4$: $k$ is the number of blocks $\Theta(cd)=\Theta(d^2)$ times the block size $\Theta(G)=\Theta(d^2)$.
\item \emph{Transfer (Section~\ref{sec:transfer}).} A cover of $\Fblr$ by $N_\rho\le(1+2D/\rho)^d(1+8D/\rho)^d$ value-based pieces, built from infimal-convolution approximations that keep the selected minimiser exact (Lemma~\ref{lem:cover}), together with the information-ratio bound applied to finitely-valued proxies, gives $\BReg_n\le n\alpha+n\rho(6+4\sqrt\beta)+\sqrt{\beta n\log N_\rho/2}$ for any fixed measurable selection (Theorem~\ref{thm:transfer}); $\log N_{1/n}=O(d\log(e+nD))$ costs $\sqrt d$.
\item \emph{Assembly.} With $\alpha=\rho=1/n$, $\sqrt\beta=O(d^4\sqrt{\log n})$ and $\sqrt{\log N_{1/n}}=O(\sqrt{d\log(e+nD)})$ give $\BReg_n=O(d^4\sqrt{dn}\log(e+nd\max\{1,D\}))=\tilde O(d^{9/2}\sqrt n)$ (Theorem~\ref{thm:main}).
\end{enumerate}

\section{Structure of uninformative rows}\label{sec:structure}

\begin{lemma}[Ridge representation with a global minimiser]\label{lem:rep}
For every $f\in\Fblr$ there exist a unit vector $\theta$ and a convex, $1$-Lipschitz $\ell:\R\to\R$ such that $f=\ell(\ip{\cdot}{\theta})$ on $K$, $\ell\ge f^\star\ge0$ on all of $\R$, and $\ell(\ip{x}{\theta})=f^\star=\min_\R\ell$ for every minimiser $x$ of $f$ on $K$. In particular $|f(x)-f(y)|\le|\ip{\theta}{x-y}|$ for all $x,y\in K$.
\end{lemma}
The proof (Appendix~\ref{app:a}) shows that the link is $1$-Lipschitz on the projection interval of $K$ and extends it beyond that interval with slopes $\mp1$; a constant extension would \emph{not} be convex when the minimiser is interior. The extension is defined on all of $\R$ and attains its global minimum at the projections of the minimisers of $f$; the part of it outside the projection interval is used in Section~\ref{sec:transfer}.

\begin{lemma}[Two thin extreme bands]\label{lem:bands}
Let $c\ge12$ and let $C$ be a configuration (Definition~\ref{def:uninf}). Fix $f\in C$ with representation $(\theta_f,\ell_f)$ from Lemma~\ref{lem:rep}, and write $s=\ip{x_f}{\theta_f}$, $t_g=\ip{x_g}{\theta_f}$, $a=\min_{g\in C}t_g$, $b=\max_{g\in C}t_g$, $L=s-a$, $R=b-s$, and
\[
\eta:=\frac{3\delta}{\varepsilon/2+\delta}=\frac{3}{c(d+1)/2+1}.
\]
If the row $f$ is uninformative then for every $g\ne f$:
\begin{enumerate}[label=(\roman*),leftmargin=2em]
\item $|t_g-s|>\varepsilon/2-2\delta\ge 5\varepsilon/12$;
\item $t_g\in(b-\eta R,\,b]$ if $t_g>s$, and $t_g\in[a,\,a+\eta L)$ if $t_g<s$.
\end{enumerate}
\end{lemma}
The proof (Appendix~\ref{app:a}) is a two-line convexity argument: uninformativeness pins $f(x_g)$ to within $\delta$ of $\fbar(x_g)$, which is at least $f^\star+\varepsilon/2-2\delta$, giving (i); and interpolating $\ell_f$ between its minimum at $s$ and its value at the extreme projection $b$ shows that any $t_g$ strictly inside $(s,b)$ would make $f(x_g)$ smaller than $\fbar(x_g)-\delta$ unless $t_g$ is within relative distance $\eta$ of $b$, giving (ii).
In particular, along its own direction each uninformative row sees the other minimisers at one of two projected distances, $L$ and $R$, each determined up to a relative error $\eta$.

\input{part2}
\input{part3}

\input{part4}

\end{document}

%% file: part2.tex
\section{The cardinality bound}\label{sec:size}

\subsection{Three linear-algebra tools}

\begin{lemma}[Trace--rank bound: the classical symmetric bound of {\cite[Lemma~2.2]{Alon09}} and its non-symmetric extension]\label{lem:tracerank}
For any real $N\times N$ matrix $M$, $|\tr M|^2\le\rank(M)\,\norm{M}_F^2$. Consequently, if $M_{ii}=1$ for all $i$, $|M_{ij}|\le t$ for $i\ne j$, $\rank M\le Q$ and $Qt^2<1$, then
\[
N\le\frac{Q(1-t^2)}{1-Qt^2}\le\frac{Q}{1-Qt^2};
\]
and if moreover $Q\ge1$ is an integer and $t<1/Q$, then $N\le Q$.
\end{lemma}
For symmetric $M$ the second claim is~\cite[Lemma~2.2]{Alon09}; the trace--rank inequality for arbitrary real matrices also appears as~\cite[Fact~10]{RVR16}. The short proof (via the orthogonal projection onto the column space of $M$) is in Appendix~\ref{app:b}.

\begin{lemma}[Hadamard products]\label{lem:hadamard}
For real matrices $U,V$ of the same shape, $\rank(U\circ V)\le\rank(U)\rank(V)$, where $\circ$ is the entrywise product.
\end{lemma}
\begin{proof}
Write $U=\sum_{a\le p}u_av_a^{\top}$ and $V=\sum_{b\le q}y_bz_b^{\top}$ with $p=\rank U$, $q=\rank V$; then $U\circ V=\sum_{a,b}(u_a\circ y_b)(v_a\circ z_b)^{\top}$.
\end{proof}

\begin{lemma}[Integer rounding preserves rank up to a factor two]\label{lem:round}
Let $r\ge1$ be an integer, let $A$ be a real matrix with $\rank A\le r$, and let $B$ be a $0$-$1$ matrix of the same shape with $\norm{A-B}_{\max}\le\gamma$. If $\gamma\le1/(4r)$ then $\rank B\le 2r-1$.
\end{lemma}
\paragraph{Proof sketch.}
If $\rank B\ge2r$, take a non-singular $2r\times2r$ submatrix $B_0$ of $B$, the corresponding $A_0$, and $E=B_0-A_0$, so $\rank A_0\le r$ and $\norm{E}_F^2\le4r^2\gamma^2$. On an $r$-dimensional subspace of $\ker A_0$ the matrix $B_0$ coincides with $E$, so the $r$ smallest squared singular values of $B_0$ sum to at most $4r^2\gamma^2$, while the $r$ largest sum to at most $\norm{B_0}_F^2\le4r^2-1$ (a non-singular $0$-$1$ matrix has a zero entry). The arithmetic--geometric mean inequality on each group gives $|\det B_0|\le(4r\gamma)^r(1-1/(4r^2))^{r/2}<1$ for $\gamma\le1/(4r)$, contradicting $|\det B_0|\ge1$ for a non-singular integer matrix. The details are in Appendix~\ref{app:b}.

We have not found this elementary statement in the literature in this form, but we do not claim priority for it; the approximate-rank literature (e.g.~\cite{Alon09,ALSV13}) studies both upper and lower bounds on the $\varepsilon$-approximate rank of real matrices and their algorithmic applications, and we use the elementary rounding statement above only in its specific role, namely to turn a band-membership matrix, which is only approximately low-rank, into an \emph{exactly} low-rank $0$-$1$ mask; the integrality of the mask is what allows the two bands of Lemma~\ref{lem:bands} to be handled without any control on the ratio of their distances. The tolerance $1/(4r)$ cannot be raised to order $1/\sqrt r$ for arbitrary $0$-$1$ matrices (Proposition~\ref{prop:hadamard} in Appendix~\ref{app:e}).

\subsection{Proof of Theorem~\ref{thm:size}}

Let $h=d+1$, $c=96h$, so that
\begin{equation}\label{eq:eta}
\eta=\frac{3}{48h^2+1}<\frac1{16h^2},\qquad\delta=\frac{\varepsilon}{96h^2}.
\end{equation}
Let $C$ be an uninformative configuration with $N=|C|\ge3$ (for $N\le2$ there is nothing to prove). For each row $i\in C$ take the representation of Lemma~\ref{lem:rep} and the quantities $t_{ij}=\ip{x_j}{\theta_i}$, $s_i$, $a_i$, $b_i$, $L_i$, $R_i$ of Lemma~\ref{lem:bands}. By Lemma~\ref{lem:bands}(i), $L_i+R_i>0$. Replacing $(\theta_i,\ell_i)$ by $(-\theta_i,\ell_i(-\cdot))$ if necessary (this changes neither the function nor its selected minimiser), we may assume
\[
0\le L_i\le R_i,\qquad R_i>0,\qquad w_i:=L_i+R_i,\qquad\lambda_i:=L_i/w_i\in[0,1/2].
\]
By Lemma~\ref{lem:bands}(ii) every $j\ne i$ is either in the \emph{near band} $t_{ij}\in[a_i,a_i+\eta L_i)$ (empty if $L_i=0$) or in the \emph{far band} $t_{ij}\in(b_i-\eta R_i,b_i]$. Throughout, $a_i,b_i,s_i,L_i,R_i$ refer to the whole configuration $C$ and are never recomputed for subsets. Split
\[
I_B=\{i:\lambda_i>1/(8h)\},\qquad I_U=\{i:\lambda_i\le1/(8h)\}.
\]

\paragraph{Balanced rows.}
For $i\in I_B$ we have $L_iR_i>0$ and $R_i/L_i<8h-1$. Consider the quadratic polynomial $p_i(x)=(\ip{x}{\theta_i}-a_i)(b_i-\ip{x}{\theta_i})/(L_iR_i)$ and the matrix $P=(p_i(x_j))_{i,j\in I_B}$. Then $P_{ii}=1$, and for $j\ne i$ in the near band $0\le P_{ij}<\eta(L_i+R_i)/R_i\le2\eta$, in the far band $0\le P_{ij}<\eta(L_i+R_i)/L_i<8h\eta$; so all off-diagonal entries are below $t_B:=8h\eta<1/(2h)$. Each row of $P$ is a polynomial of total degree at most two in $x_j\in\R^d$, so $\rank P\le Q_B:=1+d+d(d+1)/2=h(h+1)/2$. Since $Q_Bt_B^2<(h+1)/(8h)\le3/16$, Lemma~\ref{lem:tracerank} gives
\begin{equation}\label{eq:IB}
|I_B|\le\frac{Q_B}{1-Q_Bt_B^2}<\frac{4h^2(h+1)}{7h-1}<h^2,\qquad\text{i.e. }|I_B|\le h^2-1,
\end{equation}
where the last strict inequality uses $7h-1-4(h+1)=3h-5>0$ for $h\ge2$.

\paragraph{Unbalanced rows.}
All matrices below are indexed by $I_U\times I_U$. Define the $0$-$1$ matrix $H$ by $H_{ii}=1$ and, for $j\ne i$, $H_{ij}=1$ if $t_{ij}$ lies in the near band of row $i$ and $H_{ij}=0$ if it lies in the far band. Define the affine evaluation matrix $A_{ij}=(b_i-t_{ij})/w_i$; each row is an affine function of $x_j$, so $\rank A\le h$. Entrywise, $|A_{ii}-1|=\lambda_i\le1/(8h)$; for $j$ in the near band $|A_{ij}-1|=(t_{ij}-a_i)/w_i<\eta\lambda_i$; for $j$ in the far band $|A_{ij}|=(b_i-t_{ij})/w_i<\eta(1-\lambda_i)$. By~\eqref{eq:eta}, $\norm{A-H}_{\max}\le1/(8h)$, so Lemma~\ref{lem:round} with $r=h$ gives $\rank H\le2h-1$.

Next define $U_{ij}=(t_{ij}-a_i)/L_i$ if $L_i>0$ and $U_{ij}=1$ if $L_i=0$; again each row is affine in $x_j$, so $\rank U\le h$. Let $M=U\circ H$. Then $M_{ii}=1$; for $j$ in the far band $M_{ij}=0$ exactly (the mask removes the possibly huge values $U_{ij}$ there); for $j$ in the near band $0\le M_{ij}<\eta$; and rows with $L_i=0$ have no near band, so their off-diagonal entries vanish. By Lemma~\ref{lem:hadamard}, $\rank M\le h(2h-1)=:Q_U$, and $Q_U\eta<h(2h-1)/(16h^2)<1/8$, so $\eta<1/Q_U$ and the integer form of Lemma~\ref{lem:tracerank} gives
\begin{equation}\label{eq:IU}
|I_U|\le Q_U=2h^2-h.
\end{equation}

\paragraph{Conclusion.}
By~\eqref{eq:IB} and~\eqref{eq:IU}, $|C|=|I_B|+|I_U|\le3h^2-h-1<4h^2$. \qed

Only Lemmas~\ref{lem:rep} and~\ref{lem:bands} and finite-dimensional linear algebra are used: no monotonicity, no John ellipsoid, no bound on $D/\varepsilon$, and only the fact that $x_f$ is some minimiser, so the bound holds for every measurable selection. On the family of Proposition~\ref{prop:lower} every row is unbalanced, $H$ is the ``same base vertex'' mask, and $M$ is the identity.

\section{From the cardinality bound to the information ratio}\label{sec:ir}

\begin{lemma}[Counting informative pairs]\label{lem:count}
Fix $c\ge12$ and suppose every uninformative configuration (Definition~\ref{def:uninf}, with this $c$) has at most $G$ elements. Let $q\ge1$ be an integer. If $C$ satisfies the hypotheses of Definition~\ref{def:uninf} and $|C|\ge G+q$, then
\[
\sum_{(f,g)\in\PAIR(C)}\big(f(x_g)-\fbar(x_g)\big)^2\ge q\delta^2.
\]
\end{lemma}
The proof (Appendix~\ref{app:b}) deletes, $q$ times, the first endpoint of an ordered pair with $|f(x_g)-\fbar(x_g)|\ge\delta$; such a pair exists while more than $G$ elements remain, the hypotheses of Definition~\ref{def:uninf} survive deletion, and the $q$ recorded pairs have distinct first endpoints.

\begin{proof}[Proof of Theorem~\ref{thm:ir}]
Fix a measurable selection, a prior $\xi$ on $\Fblr$, and $\fbar=\E_\xi f$. Let $h=d+1$, $c=96h$ and $G=4h^2$; by Theorem~\ref{thm:size} every uninformative configuration for this $c$ has at most $G$ elements. With $\varepsilon_i=\alpha2^i$ partition $\Fblr$ into $\mathcal F_0=\{f:\fbar(x_f)-f^\star\le\alpha\}$ and $\mathcal F_i=\{f:\fbar(x_f)-f^\star\in(\alpha2^{i-1},\alpha2^i]\}$ for $1\le i\le\lceil\log_2(1/\alpha)\rceil$; these are measurable (they only involve $f\mapsto\fbar(x_f)$ and $f\mapsto f^\star$) and there are $m_\alpha$ of them. The condition of Lemma~\ref{lem:decomp} holds trivially on $\mathcal F_0$. Fix $i\ge1$, $\varepsilon=\varepsilon_i$, $\delta=\varepsilon/(ch)$, and set
\[
b=4ch,\qquad s=2G,\qquad k=bs=8chG=3072\,h^4
\]
(over $b=2ch+u$, $s=G+v$ the argument below yields energy $uv\delta^2$, and $k^2/(uv)$ is minimised at $u=2ch$, $v=G$).
Let $f_1,\dots,f_k\in\mathcal F_i$ be a $k$-tuple. We show
\begin{equation}\label{eq:block}
S:=\sum_{(j,l)\in\PAIR(k)}\big(f_j(x_{f_l})-\fbar(x_{f_l})\big)^2\ge\frac{\varepsilon^2}{12},
\end{equation}
which with $\sup_{f\in\mathcal F_i}(\fbar(x_f)-f^\star)\le\varepsilon\le\sqrt{12S}$ and Lemma~\ref{lem:decomp} (applied with $\beta_0=12$) proves the theorem. If two functions of the tuple share the selected minimiser (in particular if they coincide), the pair contributes $(\fbar(x_f)-f^\star)^2\ge\varepsilon^2/4$. Otherwise order the tuple so that $v_j=\fbar(x_{f_j})$ is non-increasing. If $v_1-v_k\ge2\varepsilon$ then $f_1(x_{f_k})\ge f_1^\star\ge v_1-\varepsilon\ge v_k+\varepsilon$ and the pair $(1,k)$ contributes $\varepsilon^2$. Otherwise cut the ordered tuple into $b$ consecutive blocks of $s$ functions; the sum of the value spreads of the blocks is less than $2\varepsilon$, so fewer than $2\varepsilon/\delta=2ch$ blocks have spread larger than $\delta$, and at least $b-2ch=2ch$ blocks have spread at most $\delta$. Each such block, together with $\fbar$, satisfies Definition~\ref{def:uninf} (its elements have $r_f\in(\varepsilon/2,\varepsilon]$ and distinct minimisers) and has $s=2G=G+G$ elements, so Lemma~\ref{lem:count} with $q=G$ shows that it contributes at least $G\delta^2$; the blocks are disjoint, so
\[
S\ge2ch\cdot G\delta^2=\frac{2G}{ch}\,\varepsilon^2=\frac{8h^2}{96h^2}\,\varepsilon^2=\frac{\varepsilon^2}{12}.
\]
Since $\varepsilon^2/4$ and $\varepsilon^2$ are not smaller, \eqref{eq:block} holds for every tuple. Lemma~\ref{lem:decomp} gives $(\alpha,12k(k-1)m_\alpha)\in\IR(\Fblr)$, and $12k^2=12\cdot3072^2h^8=113{,}246{,}208\,h^8$. Nothing in the argument depends on which measurable selection was fixed.
\end{proof}

In~\cite{BLS25} the role of Lemma~\ref{lem:count} is played by their Lemma~9, which iterates their John-ellipsoid dichotomy (Lemma~8); Proposition~\ref{prop:john} shows that no such dichotomy holds for $\Fblr$, which is why we count through a cardinality bound instead.

%% file: part3.tex
\section{From the information ratio to regret under a fixed selection rule}\label{sec:transfer}

The transfer theorem of~\cite[Theorem~28]{BLS25} bounds the regret of (approximate) TS from an information-ratio bound and a cover of the function class, and its covering argument for ridge classes uses, as the authors note, that ties among minimisers are broken consistently. For $\Fblr$ minimisers can be non-unique (an interior minimising projection level produces a positive-dimensional minimising section when $d\ge2$), so we give a self-contained transfer theorem for exact TS that is valid for any fixed measurable selection rule. Two ingredients differ from~\cite{BLS25}: the cover is built from infimal-convolution approximations, so that the selected minimiser of $f$ is an exact minimiser of its approximation (Lemma~\ref{lem:cover}); and the information-ratio hypothesis is required uniformly over selection rules and extended to finitely-supported randomised minimisers (Lemma~\ref{lem:kernel}), which is what the proof actually uses. The three preparatory lemmas below are proved in Appendix~\ref{app:c}.

\begin{lemma}[Uniform bounds extend to randomised minimisers]\label{lem:kernel}
Suppose $(\alpha,\beta)\in\IR(\mathcal F)$ for every measurable selection, where $\mathcal F\subseteq\Fblr$. Let $\xi\in\mathcal P(\mathcal F)$ and let $(f,X)$ be a random pair with $f\sim\xi$ and $X\in\argmin_K f$ almost surely, given by a measurable kernel. Then $\Delta(\mathcal L(X),\xi)\le\alpha+\sqrt{\beta I(\mathcal L(X),\xi)}$.
\end{lemma}
Lemma~\ref{lem:kernel} is where uniformity over selections is used: a randomised minimiser is a mixture of measurable selections (by Kallenberg's randomisation lemma), and both $\Delta$ and $I$ are linear in the policy.

\begin{lemma}[Continuity and near-minimisers]\label{lem:cont}
\begin{enumerate}[label=(\alph*),leftmargin=2em]
\item Let $f,g$ be Borel random elements of $C(K)$ with values in $[0,1]$ (no convexity or ridge assumption), coupled with $\norm{f-g}_\infty\le\rho$ a.s., and let $\pi$ be independent of them. Then $\sqrt{I(\pi,\mathcal L(g))}\le\sqrt{I(\pi,\mathcal L(f))}+\rho$.
\item Let $\nu$ be the law of a Borel random element of $C(K)$ that is a.s.\ $1$-Lipschitz with values in $[0,1]$, and let actions $X,Z$ be coupled with $\norm{X-Z}\le\rho$ a.s.\ and independent of the function. Then $\sqrt{I(\mathcal L(Z),\nu)}\le\sqrt{I(\mathcal L(X),\nu)}+\rho$.
\item Suppose $(\alpha,\beta)\in\IR(\Fblr)$ uniformly over selections. Let $(f,X)$ be a random pair taking finitely many values, with $f\in\Fblr$ and $f(X)\le f^\star+\rho$ a.s. Then $\Delta(\mathcal L(X),\mathcal L(f))\le\alpha+\sqrt{\beta I(\mathcal L(X),\mathcal L(f))}+\rho(1+\sqrt\beta)$.
\end{enumerate}
\end{lemma}
Parts (a) and (b) say that $\sqrt{I}$ is $1$-Lipschitz under uniform perturbations of the function and, for Lipschitz functions, of the action; both are proved with a single conditional-centring operator, which is an $L^2$ contraction. Part (c) passes from near-minimisers to exact minimisers by flattening $f$ at the level $f(X)$; since $I(\pi,\nu)$ depends only on the marginal laws of the policy and of the random function, part (a) can then be applied on a product coupling with an independent copy of $X$, even though $X$ and $f$ are dependent. Finite support is assumed so that the flattened proxy is a finitely-valued, hence $\mathcal A$-measurable, kernel.

\begin{lemma}[Measurable representations]\label{lem:meas}
There are a Borel map $f\mapsto\theta_f\in S^{d-1}$ on $(\Fblr,\mathcal B)$ and a jointly continuous map $(f,\theta,t)\mapsto L(f,\theta,t)$ on $\Fblr\times S^{d-1}\times\R$,
\[
L(f,\theta,t)=\min_{y\in K}\big[f(y)+|t-\ip{\theta}{y}|\big],
\]
such that, for every $f$, $(\theta_f,L(f,\theta_f,\cdot))$ is a representation of $f$ as in Lemma~\ref{lem:rep}.
\end{lemma}
The direction map exists by the Kuratowski--Ryll-Nardzewski selection theorem, applied to the closed-valued, weakly measurable correspondence $f\mapsto\{\theta:|f(x)-f(y)|\le|\ip{\theta}{x-y}|\ \forall x,y\}$ on the compact set $\Fblr\subset C(K)$; $L$ is the infimal convolution of $f$ with $|t-\ip{\theta}{\cdot}|$.

\begin{lemma}[Cover by infimal-convolution approximations]\label{lem:cover}
Let $0<\rho<1$, and choose a $\rho$-net $C_K$ of $K$ and a $\rho/(2D)$-net $C_S$ of $S^{d-1}$ with $|C_K|\le(1+2D/\rho)^d$ and $|C_S|\le(1+8D/\rho)^d$ (such nets exist by the standard volume argument), so $N_\rho:=|C_K||C_S|\le(1+2D/\rho)^d(1+8D/\rho)^d$. For $(x,\theta)\in C_K\times C_S$ let
\[
\mathcal F^{\mathrm{val}}_{x,\theta}=\{g\in\Fblr:\ g\text{ is a ridge function in direction }\theta,\ g(x)\le g^\star+2\rho\}.
\]
Then (i) every probability average of functions in $\mathcal F^{\mathrm{val}}_{x,\theta}$ lies in $\Fblr$, is a ridge function in direction $\theta$, and is $2\rho$-near-minimal at $x$; (ii) for every measurable selection there are measurable maps $f\mapsto\kappa(f)=(x_\kappa,\theta_\kappa)\in C_K\times C_S$ and $f\mapsto g_f\in\mathcal F^{\mathrm{val}}_{\kappa(f)}$ with $\norm{f-g_f}_\infty\le\rho/2$, $\norm{x_f-x_\kappa}\le\rho$, and such that $x_f$ is an exact minimiser of $g_f$ with $g_f^\star=f^\star$.
\end{lemma}
\paragraph{Proof idea.}
(i) is closed under averaging because all functions in a piece share the direction $\theta$, and near-minimality at $x$ is preserved since the minimum of an average is at least the average of the minima. For (ii), let $\theta_f$ and $L$ be as in Lemma~\ref{lem:meas}, let $\kappa(f)$ be the first pair in a fixed enumeration of $C_K\times C_S$ with $\norm{x_\kappa-x_f}\le\rho$ and $\norm{\theta_\kappa-\theta_f}\le\rho/(2D)$, and set $g_f(x)=L(f,\theta_\kappa,\ip{x}{\theta_\kappa})=\min_{y\in K}[f(y)+|\ip{\theta_\kappa}{x-y}|]$, the infimal convolution of $f$ with the ridge kernel $|\ip{\theta_\kappa}{\cdot}|$. Its profile is convex and $1$-Lipschitz, so $g_f$ is a ridge function in direction $\theta_\kappa$; $y=x$ gives $g_f\le f$ and $f\ge f^\star$ gives $g_f\ge f^\star$, so no rescaling is needed; the inequality $f(x)-f(y)\le|\ip{\theta_f}{x-y}|$ of Lemma~\ref{lem:rep} and $\norm{\theta_\kappa-\theta_f}\le\rho/(2D)$ give $f-\rho/2\le g_f$; and $f^\star\le g_f(x_f)\le f(x_f)=f^\star$ shows that $x_f$ is an exact minimiser of $g_f$. The details, including measurability, are in Appendix~\ref{app:c}.

\begin{theorem}[Transfer under any fixed measurable selection rule]\label{thm:transfer}
Suppose $(\alpha,\beta)\in\IR(\Fblr)$ uniformly over measurable selections. Then for every measurable selection, every prior $\xi$ on $\Fblr$, every $n\ge1$ and every $\rho\in(0,1)$, exact Thompson sampling satisfies
\[
\BReg_n(\TS,\xi)\le n\alpha+n\rho(6+4\sqrt\beta)+\sqrt{\tfrac12\beta n\log N_\rho}.
\]
\end{theorem}
\paragraph{Proof idea.}
Fix a round, let $f\sim\xi$ be the current posterior, $X^\star=x_f$, and let $X$ be an independent copy of $X^\star$, so $\mathcal L(X)=\pi_{\TS}$. With $\kappa=\kappa(f)$, $g=g_f$ and $z=x_\kappa$ from Lemma~\ref{lem:cover}, set $H_\kappa=\E[f\mid\kappa]$ and $G_\kappa=\E[g\mid\kappa]$. Both take finitely many values (one per label). $G_\kappa$ is a finitely-valued $(\Fblr,\mathcal A)$-valued random element, a ridge function in direction $\theta_\kappa$ that is $2\rho$-near-minimal at $z$ (Lemma~\ref{lem:cover}(i)); $H_\kappa$ is only a convex $1$-Lipschitz function in $\Fbl$ (an average of ridge functions with nearby but different directions is in general not a ridge function), viewed as a Borel $C(K)$-valued element, and $\norm{G_\kappa-H_\kappa}_\infty\le\rho/2$. The information-ratio hypothesis is applied, via Lemmas~\ref{lem:kernel} and~\ref{lem:cont}(c), only to the finitely-valued flattened proxy $\max\{G_\kappa,G_\kappa(z)\}$, never to $H_\kappa$. Since $g\le f$, $\norm{f-g}_\infty\le\rho/2$ and $\norm{X^\star-z}\le\rho$, the one-step regret satisfies $\Delta(\pi_{\TS},\xi)\le\Delta(\pi_z,\nu_G)+\tfrac52\rho$ with $\nu_G=\mathcal L(G_\kappa)$ and $\pi_z=\mathcal L(z)$; Lemma~\ref{lem:cont}(c) gives $\Delta(\pi_z,\nu_G)\le\alpha+\sqrt{\beta I(\pi_z,\nu_G)}+2\rho(1+\sqrt\beta)$; and Lemma~\ref{lem:cont}(a),(b), applied on a product space built from two independent posterior draws, give $\sqrt{I(\pi_z,\nu_G)}\le\sqrt{I(\pi_{\TS},\mathcal L(H_\kappa))}+\tfrac32\rho$. Altogether $\Delta(\pi_{\TS},\xi)\le\alpha+\sqrt{\beta I(\pi_{\TS},\mathcal L(H_\kappa))}+\rho(6+4\sqrt\beta)$. For Bernoulli observations $I(\pi_{\TS},\mathcal L(H_\kappa))\le\tfrac12\mathsf I(\kappa;X,Y)$ by Pinsker's inequality; the label $\kappa(f)$ is a fixed function of the environment, so the chain rule for mutual information gives $\E\sum_t\mathsf I_{t-1}(\kappa;X_t,Y_t)\le\log N_\rho$, and Cauchy--Schwarz over rounds finishes. The full proof is in Appendix~\ref{app:c}.

\paragraph{Proof of Theorem~\ref{thm:main}.}
For $n=1$ the regret is at most $1\le7$. For $n\ge2$ take $\alpha=\rho=1/n$ and $L=\log(e+nd\max\{1,D\})$: Theorem~\ref{thm:ir} gives $\sqrt\beta\le12{,}288\sqrt3\,(d+1)^4\sqrt L$ (using $m_{1/n}\le4L$), Lemma~\ref{lem:cover} gives $\log N_{1/n}\le8dL$, and Theorem~\ref{thm:transfer} yields $\BReg_n\le7+6\sqrt\beta\sqrt{ndL}=7+73{,}728\sqrt3\,(d+1)^4d^{1/2}\sqrt n\,L$; the arithmetic is in Appendix~\ref{app:c}.

\section{The lower-bound family and the failure of the John dichotomy}\label{sec:lower}

The proofs of Propositions~\ref{prop:lower} and~\ref{prop:john} are in Appendix~\ref{app:d}; we describe the constructions. For Proposition~\ref{prop:lower} take the standard $d$-simplex with vertices $v_0,\dots,v_d$, a small $\tau=1/(4c(d+1))$, and the $d(d+1)$ points $x_{ij}=(1-\tau)v_i+\tau v_j$ ($i\ne j$), which are the vertices of a truncated simplex $K$. The loss $f_{ij}$ is a V-shaped ridge function of the affine coordinate $u_{ij}=1-\lambda_i+\tau\lambda_j$ (barycentric coordinates $\lambda$), with a steep left slope and a shallow right slope, minimised at $x_{ij}$ and equal to the constant comparison function $\fbar\equiv1/4$ at every other point up to a cross error below $\delta$; the gap is $\varepsilon=\tau^2/(4d)$. Along its own direction, row $(i,j)$ sees the $d-1$ points $x_{il}$ with the same base vertex in a near band at $u$-distance $\tau^2$ and the other $d^2$ points in a far band at $u$-distance about $1$; every row is two-sided, and the ratio of the band distances is of order $1/\tau^2$, unbounded in the family. For Proposition~\ref{prop:john} take $d=3$ and the regular tetrahedron: the twelve points $x_{ij}$ are the vertices of a truncated tetrahedron $P$, removing any one vertex cuts $P$ by a plane through its three neighbours, and both $P$ and the cut body still contain the inscribed ball of the tetrahedron and are contained in the tetrahedron enlarged by the smoothing radius; an arithmetic--geometric mean bound on the determinant of any inscribed ellipsoid shows that the maximum-volume inscribed ellipsoid of both bodies is that ball, so no single removal shrinks it.

In the monotone case an uninformative row has all other minimisers on one side of its own along the ridge direction, and~\cite{BLS25} show that they then lie beyond the centre of the John ellipsoid, so removing the row cuts the ellipsoid through its centre; for non-monotone links both bands of Lemma~\ref{lem:bands} can be non-empty and arbitrarily imbalanced, and Proposition~\ref{prop:john} shows that the single-removal, fixed-shrink-factor dichotomy cannot hold. We do not obtain the required uniform cardinality estimate from polynomial approximation (a univariate polynomial that is $1$ at the row's own projection and small on both bands); the mask argument of Lemma~\ref{lem:round} avoids the need to control the imbalance ratio of the two band distances.

\section{Discussion}\label{sec:discussion}

Theorem~\ref{thm:main} resolves the qualitative question of~\cite{BLS25}: monotonicity of the link is not needed for exact-posterior Thompson sampling to achieve Bayesian regret polynomial in $d$ and $\sqrt n$ in $n$. The quantitative question of the dimension dependence remains open: the information-ratio certificate we construct has size $O((d+1)^8(1+\log(1/\alpha)))$ against $\tilde O(d^4)$ in the monotone analysis (this is the size of our certificate, not a lower bound on the best information ratio), and the resulting regret exponent is $9/2$ against $5/2$. The loss comes from two places. The cardinality bound $\Theta(d^2)$ is tight (Proposition~\ref{prop:lower}), so the number of functions per block in Lemma~\ref{lem:decomp} is $\Theta(d^2)$ and the number of blocks is $\Theta(cd)$; and the constant $c=96(d+1)$ is dictated by the rounding threshold $1/(4r)$ of Lemma~\ref{lem:round}, which forces $\eta\lesssim1/d^2$. The order $1/r$ of this tolerance cannot be improved to order $1/\sqrt r$ for arbitrary $0$-$1$ matrices: Proposition~\ref{prop:hadamard} (Appendix~\ref{app:e}) exhibits, for every $r=2^j-1$ with $j\ge2$, a rank-$r$ matrix within $2/(r+1)$ in max-norm of a $0$-$1$ matrix of rank $2r$. This rules out only a better uniform rounding theorem for arbitrary $0$-$1$ matrices; the band-membership matrices of Section~\ref{sec:size} have additional structure, and any improvement of the exponent through a larger rounding tolerance would have to exploit it. We do not claim that the exponent $9/2$ is optimal. Whether the monotone exponent $5/2$ can be reached, and what the correct dimension dependence of TS on this class is (no lower bound specific to non-monotone links is known), are left open. Three remarks on scope. (1) The statement is about Bayesian regret with an arbitrary prior; it does not give a frequentist guarantee for every fixed environment, and it concerns exact-posterior TS without addressing computation, which remains open. (2) The selection rule is arbitrary but fixed in advance; rules that change with the history are not covered. (3) The observation model is Bernoulli; the only place the model enters is Pinsker's inequality for the conditional means of $[0,1]$-valued observations, so the same constants hold for observations $Y_t\in[0,1]$ satisfying $\E[Y_t\mid\mathcal H_{t-1},X_t,f]=f(X_t)$ almost surely, where $\mathcal H_{t-1}=\sigma(X_1,Y_1,\dots,X_{t-1},Y_{t-1})$, provided that a complete observation model is specified under which the posterior is well defined and TS uses fresh internal randomness conditionally on the history (the condition $\E[Y_t\mid X_t,f]=f(X_t)$ alone is not sufficient). The second question in the same paragraph of~\cite{BLS25}, the information ratio for a \emph{known} convex link, is not touched here; see also the MSc thesis of Bakhtiari~\cite{BakhtiariThesis}.

%% file: part4.tex
\appendix

\section{Proofs for Sections~\ref{sec:setting} and~\ref{sec:structure}}\label{app:a}

\begin{proof}[Proof of Lemma~\ref{lem:decomp}]
Let $w_i=\xi(\mathcal F_i)$, discard pieces with $w_i=0$, let $\xi_i$ be the conditional law on $\mathcal F_i$, and $J_i=\E_{\xi_i\otimes\xi_i}[(f(x_g)-\fbar(x_g))^2]$. Taking expectations of the hypothesis over $k$ independent draws from $\xi_i$ gives $\sup_{f\in\mathcal F_i}(\fbar(x_f)-f^\star)\le\alpha+\sqrt{\beta_0k(k-1)J_i}$. Hence
\begin{align*}
\Delta(\pi^\xi_{\TS},\xi)&=\sum_iw_i\,\E_{\xi_i}[\fbar(x_f)-f^\star]\le\alpha+\sqrt{\beta_0k(k-1)}\sum_iw_i\sqrt{J_i}\\
&\le\alpha+\sqrt{\beta_0k(k-1)m\textstyle\sum_iw_i^2J_i}\le\alpha+\sqrt{\beta_0k(k-1)m\,I(\pi^\xi_{\TS},\xi)},
\end{align*}
by Cauchy--Schwarz and because $I(\pi^\xi_{\TS},\xi)=\sum_{i,j}w_iw_jJ_{ij}\ge\sum_iw_i^2J_i$ with $J_{ij}=\E_{\xi_i\otimes\xi_j}[(f(x_g)-\fbar(x_g))^2]\ge0$.
\end{proof}

\begin{proof}[Proof of Lemma~\ref{lem:rep}]
If $f$ is constant, any unit $\theta$ and the constant link work. Otherwise normalise the ridge direction and let $u$ be the link on $I=\{\ip{x}{\theta}:x\in K\}=[a,b]$. For $t\in(a,b)$ there is $z\in\mathrm{int}\,K$ with $\ip{z}{\theta}=t$ and a short segment through $z$ in direction $\theta$ inside $K$, along which $f$ is $1$-Lipschitz, so $u$ is locally $1$-Lipschitz on $(a,b)$ and, by continuity of convex functions, $1$-Lipschitz on $I$. Define $\ell(t)=u(a)+a-t$ for $t<a$, $\ell=u$ on $I$, and $\ell(t)=u(b)+t-b$ for $t>b$. The slopes $-1$ on the left and $+1$ on the right bracket every secant slope of $u$, so $\ell$ is convex and $1$-Lipschitz; outside $I$ it exceeds the endpoint values, so $\min_\R\ell=\min_I u=f^\star$, attained at the projections of the minimisers of $f$. The last claim is the $1$-Lipschitz property of $\ell$.
\end{proof}

\begin{proof}[Proof of Lemma~\ref{lem:bands}]
Uninformativeness and the comparison condition give $f(x_g)>\fbar(x_g)-\delta\ge\fbar(x_f)-2\delta=f^\star+r_f-2\delta\ge f^\star+\varepsilon/2-2\delta$; with $1$-Lipschitzness of $\ell_f$ this is (i), and $\varepsilon/2-2\delta\ge5\varepsilon/12$ because $c\ge12$. For (ii) take $g$ with $t_g>s$ and let $h$ attain $t_h=b$ (so $h\ne f$ and $b>s$). With $\lambda=(b-t_g)/(b-s)\in[0,1)$, convexity of $\ell_f$ gives $f(x_g)\le\lambda f^\star+(1-\lambda)f(x_h)$. Using $f^\star\le\fbar(x_g)+\delta-\varepsilon/2$ and $f(x_h)<\fbar(x_h)+\delta\le\fbar(x_g)+2\delta$,
\[
f(x_g)<\fbar(x_g)+2\delta-\lambda(\varepsilon/2+\delta),
\]
and combining with $f(x_g)>\fbar(x_g)-\delta$ yields $\lambda<\eta$. The left side is symmetric.
\end{proof}

\section{Proofs for Sections~\ref{sec:size} and~\ref{sec:ir}}\label{app:b}

\begin{proof}[Proof of Lemma~\ref{lem:tracerank}]
For symmetric $M$ the second claim is~\cite[Lemma~2.2]{Alon09}, which Alon attributes to earlier sources; the trace--rank inequality for arbitrary real matrices also appears as~\cite[Fact~10]{RVR16}, and we include the short argument. Let $\Pi$ be the orthogonal projection onto the column space of $M$, of rank $q=\rank M$. Then $|\tr M|=|\tr(\Pi M)|=|\ip{\Pi}{M}_F|\le\norm{\Pi}_F\norm{M}_F=\sqrt q\norm{M}_F$. Inserting $\tr M=N$ and $\norm{M}_F^2\le N+N(N-1)t^2$ gives the second claim. For the third, $(Q+1)-Q(1-t^2)/(1-Qt^2)=(1-Q^2t^2)/(1-Qt^2)>0$ when $t<1/Q$, so $N<Q+1$.
\end{proof}

\begin{proof}[Proof of Lemma~\ref{lem:round}]
Suppose $\rank B\ge2r$ and take a non-singular $2r\times2r$ submatrix $B_0$ of $B$; let $A_0$ be the corresponding submatrix of $A$ and $E=B_0-A_0$, so $\rank A_0\le r$ and $\norm{E}_F^2\le4r^2\gamma^2$. Let $\sigma_1\ge\dots\ge\sigma_{2r}>0$ be the singular values of $B_0$. Since $\dim\ker A_0\ge r$, let $\Pi$ be the orthogonal projection onto an $r$-dimensional subspace of $\ker A_0$; then $B_0\Pi=E\Pi$, so $\norm{B_0\Pi}_F^2\le\norm{E}_F^2$. In an orthonormal basis of right singular vectors of $B_0$, $\norm{B_0\Pi}_F^2=\sum_k\sigma_k^2w_k$ with $0\le w_k\le1$ and $\sum_kw_k=r$, which is at least the sum of the $r$ smallest $\sigma_k^2$. Hence
\[
\sum_{k=r+1}^{2r}\sigma_k^2\le4r^2\gamma^2,\qquad\sum_{k=1}^{r}\sigma_k^2\le\norm{B_0}_F^2\le4r^2-1,
\]
the latter because a non-singular $2r\times2r$ $0$-$1$ matrix has at least one zero entry. By the arithmetic--geometric mean inequality applied to each group of $r$ squared singular values,
\[
|\det B_0|=\prod_{k=1}^{2r}\sigma_k\le\Big(\frac{4r^2-1}{r}\cdot\frac{4r^2\gamma^2}{r}\Big)^{r/2}=(4r\gamma)^r\Big(1-\frac1{4r^2}\Big)^{r/2}<1\qquad\text{for }\gamma\le\frac1{4r},
\]
contradicting $|\det B_0|\ge1$ for a non-singular integer matrix.
\end{proof}

\begin{proof}[Proof of Lemma~\ref{lem:count}]
As long as the remaining set has more than $G$ elements it is not uninformative, so it contains an ordered pair $(f,g)$ with $|f(x_g)-\fbar(x_g)|\ge\delta$; delete $f$ (only the first endpoint). The hypotheses of Definition~\ref{def:uninf} are per-element and pairwise, so they survive deletion, and before the $(j+1)$-st deletion ($0\le j\le q-1$) at least $G+q-j\ge G+1$ elements remain. The $q$ recorded ordered pairs have distinct first endpoints, hence are distinct, and each contributes at least $\delta^2$.
\end{proof}

\section{Proofs for Section~\ref{sec:transfer}}\label{app:c}

\begin{proof}[Proof of Lemma~\ref{lem:kernel}]
By the randomisation lemma for kernels into Borel spaces~\cite[Lemma~2.22]{Kallenberg97} there is a jointly measurable $a(f,u)$ with $X=a(f,U)$ for $U$ uniform on $[0,1]$ independent of $f$. Let $a_0$ be a fixed measurable selection, $B=\{(f,u):f(a(f,u))>f^\star\}$, and $\tilde a=a$ off $B$, $\tilde a=a_0$ on $B$; then $\tilde a$ is jointly measurable, $f\mapsto\tilde a(f,u)$ is a measurable selection for every $u$, and $(f,\tilde a(f,U))$ has the same law as $(f,X)$. Let $\pi_u$ be the law of $\tilde a(f,u)$. Both $\Delta(\cdot,\xi)$ and $I(\cdot,\xi)$ are linear in the policy, so $\Delta(\mathcal L(X),\xi)=\int\Delta(\pi_u,\xi)\,du\le\alpha+\int\sqrt{\beta I(\pi_u,\xi)}\,du\le\alpha+\sqrt{\beta\int I(\pi_u,\xi)\,du}$ by concavity of the square root.
\end{proof}

\begin{proof}[Proof of Lemma~\ref{lem:cont}]
For $(X,h)\sim\pi\otimes\nu$ one has $\bar h(X)=\E[h(X)\mid X]$, so $I(\pi,\nu)^{1/2}=\norm{h(X)-\E[h(X)\mid X]}_{L^2}$. (a): realise $X\sim\pi$ independently of the coupled pair $(f,g)$ and let $QW=W-\E[W\mid X]$, an $L^2$ contraction; then $\sqrt{I(\pi,\mathcal L(f))}=\norm{Qf(X)}_{L^2}$, $\sqrt{I(\pi,\mathcal L(g))}=\norm{Qg(X)}_{L^2}$, and the two differ by at most $\norm{Q(g(X)-f(X))}_{L^2}\le\norm{g(X)-f(X)}_{L^2}\le\rho$. (b): realise $h\sim\nu$ independently of the coupled pair $(X,Z)$ and use the single centring operator $QW=W-\E[W\mid X,Z]$; independence gives $Qh(X)=h(X)-\bar h(X)$ and $Qh(Z)=h(Z)-\bar h(Z)$, so $\sqrt{I(\mathcal L(Z),\nu)}$ and $\sqrt{I(\mathcal L(X),\nu)}$ differ by at most $\norm{Q(h(Z)-h(X))}_{L^2}\le\norm{h(Z)-h(X)}_{L^2}\le\rho$ by the Lipschitz bound. (c): let $g=\max(f,f(X))$; truncation at a level in $[0,1]$ preserves convexity, range, Lipschitzness and the ridge direction, so $g\in\Fblr$; $X$ is an exact minimiser of $g$, $\norm{g-f}_\infty\le\rho$, $\fbar\le\bar g$ and $f^\star\ge g^\star-\rho$. Since $(f,X)$ takes finitely many values, $(g,X)$ is a finitely-valued measurable kernel of exact minimisers (distinct values of $(f,X)$ may share $g$ but not $X$, which is why a kernel rather than a selection is needed), and Lemma~\ref{lem:kernel} gives $\Delta(\mathcal L(X),\mathcal L(f))\le\Delta(\mathcal L(X),\mathcal L(g))+\rho\le\alpha+\sqrt{\beta I(\mathcal L(X),\mathcal L(g))}+\rho$. Finally, since $I(\pi,\nu)$ depends only on the marginal law of the policy and the marginal law of the random function, part (a) may be applied on a product coupling in which an independent $X'\sim\mathcal L(X)$ is used together with the coupled pair $(f,g)$, $\norm{g-f}_\infty\le\rho$; this gives $\sqrt{I(\mathcal L(X),\mathcal L(g))}\le\sqrt{I(\mathcal L(X),\mathcal L(f))}+\rho$ and finishes the proof.
\end{proof}

\begin{proof}[Proof of Lemma~\ref{lem:meas}]
A unit vector $\theta$ is a ridge direction of $f$ if and only if $|f(x)-f(y)|\le|\ip{\theta}{x-y}|$ for all $x,y\in K$ (the forward direction is Lemma~\ref{lem:rep}, the backward direction yields a profile that is constant on fibres, convex and $1$-Lipschitz). Let $\Phi(f)$ be the set of such $\theta$. Profiles extended as in Lemma~\ref{lem:rep} satisfy $\ell(0)=f(0)\in[0,1]$ and $\mathrm{Lip}(\ell)\le1$, so they are uniformly bounded and equicontinuous on $[-D,D]$; if $f_j\to f$ uniformly and $\theta_j\to\theta$, a uniformly convergent subsequence of profiles shows $f=\ell(\ip{\cdot}{\theta})$. Hence $\Fblr$ is compact in $C(K)$ (Arzel\`a--Ascoli) and the graph of $\Phi$ is closed in $\Fblr\times S^{d-1}$, so $\Phi$ has non-empty closed values. For an open $O\subseteq S^{d-1}$ write $O=\bigcup_jC_j$ with $C_j$ compact; then $\{f:\Phi(f)\cap O\ne\emptyset\}=\bigcup_j\mathrm{proj}(\mathrm{Graph}(\Phi)\cap(\Fblr\times C_j))$ is a countable union of compact sets, hence Borel, so $\Phi$ is weakly measurable and the Kuratowski--Ryll-Nardzewski selection theorem~\cite{KRN65} gives a Borel selection $\theta_f\in\Phi(f)$. The map $L$ satisfies $|L(f,\theta,t)-L(g,\psi,s)|\le\norm{f-g}_\infty+D\norm{\theta-\psi}+|t-s|$, hence is jointly continuous, and for $\theta\in\Phi(f)$ it coincides with the profile of $f$ on the projection interval and with the $\mp1$-slope extension of Lemma~\ref{lem:rep} outside it.
\end{proof}

\begin{proof}[Proof of Lemma~\ref{lem:cover}]
(i) Averages of convex ridge functions with a common direction are convex ridge functions with that direction; range and Lipschitz bounds are preserved; and if $g_\nu=\int g\,d\nu$ then $g_\nu(x)\le\int g^\star d\nu+2\rho\le g_\nu^\star+2\rho$ since the minimum of an average is at least the average of the minima. (ii) Let $\theta_f$ and $L$ be as in Lemma~\ref{lem:meas} and let $\kappa(f)$ be the first pair in a fixed enumeration of $C_K\times C_S$ with $\norm{x_\kappa-x_f}\le\rho$ and $\norm{\theta_\kappa-\theta_f}\le\rho/(2D)$; since $x_f$ is $\mathcal A$-measurable and $\theta_f$ is Borel, $\kappa$ is $\mathcal A$-measurable. Define
\[
g_f(x)=L(f,\theta_\kappa,\ip{x}{\theta_\kappa})=\min_{y\in K}\big[f(y)+|\ip{\theta_\kappa}{x-y}|\big].
\]
The profile $t\mapsto L(f,\theta_\kappa,t)$ is convex (the minimand is jointly convex in $(t,y)$ and partial minimisation over the convex set $K$ preserves convexity) and $1$-Lipschitz (each $y$ contributes a $1$-Lipschitz function of $t$), so $g_f$ is a convex ridge function in direction $\theta_\kappa$ with $\mathrm{Lip}(g_f)\le1$. Taking $y=x$ gives $g_f\le f\le1$, and $f\ge f^\star$ gives $g_f\ge f^\star\ge0$; hence $g_f\in\Fblr$ with no rescaling. By Lemma~\ref{lem:rep}, for all $x,y\in K$, $f(x)-f(y)\le|\ip{\theta_f}{x-y}|\le|\ip{\theta_\kappa}{x-y}|+D\norm{\theta_\kappa-\theta_f}\le|\ip{\theta_\kappa}{x-y}|+\rho/2$ (using $\norm{x-y}\le D$); minimising over $y$ gives $f(x)-\rho/2\le g_f(x)\le f(x)$, i.e.\ $\norm{g_f-f}_\infty\le\rho/2$. At the selected minimiser, $f^\star\le g_f(x_f)\le f(x_f)=f^\star$, so $x_f$ minimises $g_f$ and $g_f^\star=f^\star$; by $1$-Lipschitzness $g_f(x_\kappa)-g_f^\star\le\norm{x_\kappa-x_f}\le\rho\le2\rho$, so $g_f\in\mathcal F^{\mathrm{val}}_{\kappa(f)}$. Finally $f\mapsto g_f$ is measurable as a $(C(K),\mathcal B)$-valued map and $(f,x)\mapsto g_f(x)$ is jointly measurable, by the joint continuity of $L$ and the measurability of $\kappa$.
\end{proof}

\begin{proof}[Proof of Theorem~\ref{thm:transfer}]
We bound the one-step regret under an arbitrary posterior and then sum. Let $f\sim\xi$ (a posterior at some round), $X^\star=x_f$, and let $X$ be an independent copy of $X^\star$, so $\mathcal L(X)=\pi_{\TS}$. Let $\kappa=\kappa(f)$, $g=g_f$, $z=x_\kappa$ be as in Lemma~\ref{lem:cover}, and
\[
H_\kappa=\E[f\mid\kappa],\qquad G_\kappa=\E[g\mid\kappa],\qquad\nu_G=\mathcal L(G_\kappa),\qquad\pi_z=\mathcal L(z).
\]
Both $H_\kappa$ and $G_\kappa$ take finitely many values (one per label). $G_\kappa$ is a finitely-valued $(\Fblr,\mathcal A)$-valued random element; $H_\kappa$ is viewed only as a Borel $C(K)$-valued random element with values in $\Fbl$. By Lemma~\ref{lem:cover}(i), $G_\kappa\in\Fblr$ is a ridge function in direction $\theta_\kappa$ and is $2\rho$-near-minimal at $z$; $H_\kappa$ is merely a convex $1$-Lipschitz function $K\to[0,1]$ (an average of ridge functions with nearby but different directions is in general not a ridge function), and $\norm{G_\kappa-H_\kappa}_\infty\le\rho/2$. In this proof the information-ratio hypothesis is applied, via Lemmas~\ref{lem:kernel} and~\ref{lem:cont}(c), to the finitely-valued flattened proxy $\max\{G_\kappa,G_\kappa(z)\}$, never to $H_\kappa$; $H_\kappa$ enters only through the continuity statements of Lemma~\ref{lem:cont}(a),(b), which do not require the ridge structure.

First, $\E[f^\star\mid\kappa]=\E[f(X^\star)\mid\kappa]\ge\E[g(X^\star)\mid\kappa]\ge G_\kappa(z)-\rho\ge G_\kappa^\star-\rho$ (using $g\le f$, $\norm{X^\star-z}\le\rho$ and Lipschitzness), and $\norm{\E f-\E g}_\infty\le\rho/2$; hence
\[
\Delta(\pi_{\TS},\xi)=\E[\fbar(X)-f^\star]\le\E[\bar g(X)-G_\kappa^\star]+\tfrac32\rho\le\E[\bar g(z)-G_\kappa^\star]+\tfrac52\rho=\Delta(\pi_z,\nu_G)+\tfrac52\rho.
\]
By Lemma~\ref{lem:cont}(c) applied to the finitely-valued pair $(G_\kappa,z)$, $\Delta(\pi_z,\nu_G)\le\alpha+\sqrt{\beta I(\pi_z,\nu_G)}+2\rho(1+\sqrt\beta)$. The information terms depend only on the product of the marginal laws, so the continuity statements may be applied on a product space: let $F_1,F_2$ be independent draws from the current posterior, build $(G_1,H_1)=(G_{\kappa(F_1)},H_{\kappa(F_1)})$ from $F_1$ and $(Z_2,A_2)=(x_{\kappa(F_2)},x_{F_2})$ from $F_2$. Lemma~\ref{lem:cont}(a) with the coupling $(G_1,H_1)$, $\norm{G_1-H_1}_\infty\le\rho/2$, independent of $Z_2\sim\pi_z$, gives $\sqrt{I(\pi_z,\nu_G)}\le\sqrt{I(\pi_z,\mathcal L(H_\kappa))}+\rho/2$; Lemma~\ref{lem:cont}(b) with the coupling $(Z_2,A_2)$, $\norm{Z_2-A_2}\le\rho$, independent of $H_1\sim\mathcal L(H_\kappa)$, gives $\sqrt{I(\pi_z,\mathcal L(H_\kappa))}\le\sqrt{I(\pi_{\TS},\mathcal L(H_\kappa))}+\rho$. Together $\sqrt{I(\pi_z,\nu_G)}\le\sqrt{I(\pi_{\TS},\mathcal L(H_\kappa))}+\tfrac32\rho$. Altogether
\[
\Delta(\pi_{\TS},\xi)\le\alpha+\sqrt{\beta I(\pi_{\TS},\mathcal L(H_\kappa))}+\rho\big(\tfrac92+\tfrac72\sqrt\beta\big)\le\alpha+\sqrt{\beta I(\pi_{\TS},\mathcal L(H_\kappa))}+\rho(6+4\sqrt\beta).
\]
For Bernoulli observations, $I(\pi_{\TS},\mathcal L(H_\kappa))=\E[(\E[Y\mid X]-\E[Y\mid X,\kappa])^2]\le\tfrac12\E[\mathrm{KL}(P_{Y\mid X,\kappa}\,\|\,P_{Y\mid X})]=\tfrac12\mathsf I(\kappa;X,Y)$ by Pinsker's inequality. The label $\kappa=\kappa(f)$ is a fixed finite-valued function of the environment (it depends on the fixed nets, the fixed direction map of Lemma~\ref{lem:meas} and the fixed selection, not on the round), so summing over rounds and using the chain rule for mutual information and Cauchy--Schwarz,
\[
\BReg_n(\TS,\xi)\le n\alpha+n\rho(6+4\sqrt\beta)+\sqrt{\tfrac12\beta n\,\E\textstyle\sum_t\mathsf I_{t-1}(\kappa;X_t,Y_t)}\le n\alpha+n\rho(6+4\sqrt\beta)+\sqrt{\tfrac12\beta n\log N_\rho}.\qedhere
\]
\end{proof}

\begin{proof}[Proof of Theorem~\ref{thm:main}]
For $n=1$ the regret is at most $1\le7$. Let $n\ge2$, $\alpha=\rho=1/n$, $L=\log(e+nd\max\{1,D\})$, and $h=d+1$. From Theorem~\ref{thm:ir}, $k=3072h^4$ and $m_{1/n}\le4L$, so $\sqrt\beta\le\sqrt{12\cdot4L}\,k=4\sqrt3\cdot3072\,h^4\sqrt L=12{,}288\sqrt3\,h^4\sqrt L$. From Lemma~\ref{lem:cover}, $\log N_{1/n}\le d\log(1+2nD)+d\log(1+8nD)\le8dL$. Theorem~\ref{thm:transfer} gives $\BReg_n\le1+6+4\sqrt\beta+\sqrt\beta\sqrt{4ndL}\le7+\sqrt\beta\,(4+2\sqrt{ndL})\le7+6\sqrt\beta\sqrt{ndL}$ because $\sqrt{ndL}\ge1$, i.e.\ $\BReg_n\le7+73{,}728\sqrt3\,h^4d^{1/2}\sqrt n\,L$.
\end{proof}

\section{Proofs for Section~\ref{sec:lower}}\label{app:d}

\begin{proof}[Proof of Proposition~\ref{prop:lower}]
Let $d\ge2$, $c\ge12$, $\tau=1/(4c(d+1))$. Take the standard $d$-simplex with vertices $v_0,\dots,v_d$ and barycentric coordinates $\lambda_0,\dots,\lambda_d$ (translated so that the average of the points below is $0$). For each ordered pair $i\ne j$ let $x_{ij}=(1-\tau)v_i+\tau v_j$, and let $K=\conv\{x_{ij}\}$ (full-dimensional, containing $0$, $D\le\sqrt2$). Let $u_{ij}(x)=1-\lambda_i(x)+\tau\lambda_j(x)$, an affine function, and set $v=1/4$, $\varepsilon=\tau^2/(4d)$, $\fbar\equiv v$, $s^u=\tau+\tau^2$, and
\[
f_{ij}(x)=v-\varepsilon+\varepsilon\max\Big\{\frac{s^u-u_{ij}(x)}{\tau^2},\ \frac{u_{ij}(x)-s^u}{1-2\tau^2}\Big\},\qquad x_{f_{ij}}=x_{ij}.
\]
Each $f_{ij}$ is the maximum of two affine functions, hence a convex ridge function, minimised exactly where $u_{ij}=s^u$, in particular at $x_{ij}$, with minimum $v-\varepsilon$. Evaluating $u_{ij}$ at the $d(d+1)$ points gives exactly the values $\tau+\tau^2$ (at $x_{ij}$), $\tau$ (the $d-1$ points $x_{il}$), $1-\tau^2$ ($x_{ji}$), $1+\tau-\tau^2$ ($x_{jl}$), $1-\tau$ ($x_{ki}$), $1+\tau^2$ ($x_{kj}$) and $1$ (the remaining $(d-1)(d-2)$ points), so on $K$ the values of $f_{ij}$ lie in $[v-\varepsilon,v]$, the points $x_{il}$ have $f_{ij}=v$ exactly, and all other cross values satisfy $|f_{ij}(x_{kl})-v|\le\varepsilon E_\tau$ with $E_\tau=(2\tau-\tau^2)/(1-2\tau^2)<3\tau$. Since $\delta=\varepsilon/(c(d+1))=4\tau\varepsilon$, every ordered pair is uninformative; the gaps are $r_f=\varepsilon$ and the comparison values coincide. Lipschitzness: $\norm{\nabla u_{ij}}\le(1+\tau)\sqrt d$ in the standard simplex coordinates, so $\mathrm{Lip}(f_{ij})\le\varepsilon(1+\tau)\sqrt d/\tau^2=(1+\tau)/(4\sqrt d)<1$. The functions are pairwise distinct because only $f_{ij}$ attains $v-\varepsilon$ at $x_{ij}$. Along its own direction each row sees the $d-1$ points $x_{il}$ at distance $\tau^2$ in $u_{ij}$-coordinates (the near band; the unit-direction distances of Lemma~\ref{lem:bands} are these values divided by $\norm{\nabla u_{ij}}$) and the other $d^2$ points at $u_{ij}$-distance about $1$ (the far band). Finally $D/\varepsilon\le\sqrt2\cdot4d/\tau^2=64\sqrt2\,c^2d(d+1)^2$; for every $0<\varepsilon'\le\varepsilon$ keep $K$, the selected points and the ridge directions fixed and replace each function by $f'_{ij}=v+(\varepsilon'/\varepsilon)(f_{ij}-v)$, with comparison function $\fbar'\equiv v$: the diagonal gaps and the cross errors are multiplied by $\varepsilon'/\varepsilon$ while the Lipschitz constants do not increase, so the new configuration is uninformative at scale $\varepsilon'$ (with $\delta'=\varepsilon'/(c(d+1))$) and every larger ratio $D/\varepsilon'$ is attained. The stated parameter cost for $c=96(d+1)$ follows by substituting $\tau=1/(384(d+1)^2)$ and $\varepsilon=\tau^2/(4d)$; the upper bound is Theorem~\ref{thm:size}. For $d=1$ take $K=[0,1]$, $\fbar\equiv1/4$, $f_1=1/4-\varepsilon+\varepsilon x$, $f_2=1/4-\varepsilon+\varepsilon(1-x)$ with minimisers $0,1$ and $0<\varepsilon\le1/4$, which requires $D/\varepsilon\ge4$. The bound $|C|\le2$ in dimension one follows from Lemma~\ref{lem:bands}(ii): if $x_1<x_2<x_3$ are three minimisers of an uninformative configuration, the row of $x_1$ sees $x_2$ in its far band, $x_2\ge x_3-\eta(x_3-x_1)$, and the row of $x_3$ sees $x_2$ in its far band, $x_2\le x_1+\eta(x_3-x_1)$; together $(1-2\eta)(x_3-x_1)\le0$, contradicting $\eta<1/2$. (For $D/\varepsilon\le5/12$ not even two points fit, since Lemma~\ref{lem:bands}(i) forces a distance larger than $5\varepsilon/12$.)
\end{proof}

\begin{proof}[Proof of Proposition~\ref{prop:john}]
Take the regular tetrahedron $v_0=\tfrac12(1,1,1)$, $v_1=\tfrac12(1,-1,-1)$, $v_2=\tfrac12(-1,1,-1)$, $v_3=\tfrac12(-1,-1,1)$, so $\lambda_i(x)=\tfrac14+\ip{v_i}{x}$, and use the configuration of Proposition~\ref{prop:lower} with $K=P:=\conv\{x_{ij}\}$ (a truncated tetrahedron; the twelve points are its vertices). The inscribed ball of the tetrahedron $S=\conv\{v_i\}$ is $B_r$ with $r=1/(2\sqrt3)$, and on $B_r$ one has $0\le\lambda_i\le1/2$, so $B_r\subset P$. Fix a vertex $x_{ij}$ and let $b_0=(1-2\tau)/(1-\tau)$. The affine inequality $\lambda_i+b_0\lambda_j\le1-\tau$ is violated by $x_{ij}$, holds with equality at its three neighbours $x_{ik}$ ($k\ne i,j$) and $x_{ji}$, and holds at all other vertices; since the cutting plane meets every edge from $x_{ij}$ exactly at a neighbour, $P_{-ij}:=\conv\{x_{kl}:(k,l)\ne(i,j)\}=P\cap\{\lambda_i+b_0\lambda_j\le1-\tau\}$. On $B_r$, $\max(\lambda_i+b_0\lambda_j)=\big(1+b_0+\sqrt{1+b_0^2-2b_0/3}\big)/4\le(2+\sqrt2)/4<1-\tau$ for $\tau\le1/192$, which holds for all $c\ge12$; hence $B_r\subset P_{-ij}$ for all $i\ne j$. Now let $\zeta>0$. Both $J_\zeta(C)=P+B_\zeta$ and $J_\zeta(C\setminus\{f_{ij}\})=P_{-ij}+B_\zeta$ contain $B_{r+\zeta}$ and are contained in the enlarged tetrahedron $S_\zeta=\{x:\ip{u_k}{x}\le r+\zeta,\ k=0,\dots,3\}$, where $u_k$ are the unit outer normals of $S$, which satisfy $\sum_ku_k=0$ and $\sum_ku_ku_k^{\top}=\tfrac43I$. If an ellipsoid $z+AB_1$ ($A\succ0$ symmetric) lies in $S_\zeta$ then $\ip{u_k}{z}+\norm{Au_k}\le r+\zeta$ for each $k$; summing and using $\norm{Au_k}\ge u_k^{\top}Au_k$ gives $\tfrac43\tr A\le4(r+\zeta)$, so $\det A\le(\tr A/3)^3\le(r+\zeta)^3$. Hence the maximum-volume inscribed ellipsoid of both bodies is $B_{r+\zeta}$ (equality in the arithmetic--geometric mean inequality forces $A=(r+\zeta)I$ and then $z=0$), and every single-vertex removal leaves it unchanged.
\end{proof}

\section{A limit of max-norm rounding}\label{app:e}

The tolerance $1/(4r)$ in Lemma~\ref{lem:round} has the order $1/r$. The following proposition shows that, for arbitrary $0$-$1$ matrices, no tolerance of order $1/\sqrt r$ can work; it is used only in the discussion of Section~\ref{sec:discussion}.

\begin{proposition}[Max-norm rounding cannot have a uniform $1/\sqrt r$ tolerance]\label{prop:hadamard}
For every $r=2^j-1$ with $j\ge2$ there are $A,B\in\R^{2r\times2r}$ such that $\rank A=r$, $B$ is a $0$-$1$ matrix with $\rank B=2r$, and $\norm{A-B}_{\max}=2/(r+1)$. Consequently, for every constant $a>0$ there are infinitely many $r$ for which the implication ``$\rank A\le r$ and $\norm{A-B}_{\max}\le a/\sqrt r$ for a $0$-$1$ matrix $B$ imply $\rank B\le2r-1$'' fails.
\end{proposition}
\begin{proof}
Let $s=r+1=2^j$ and let $S_s$ be the Sylvester--Hadamard matrix of order $s$, defined by $S_1=(1)$ and $S_{2m}=\begin{psmallmatrix}S_m&S_m\\S_m&-S_m\end{psmallmatrix}$; it is symmetric with entries $\pm1$, its first row and column consist of ones, and $S_sS_s^{\top}=sI_s$. Write $S_s=\begin{psmallmatrix}1&\mathbf 1^{\top}\\ \mathbf 1&C\end{psmallmatrix}$ with $C$ a symmetric $r\times r$ matrix with entries $\pm1$. Orthogonality of each of the rows $2,\dots,s$ to the first row gives $C\mathbf 1=-\mathbf 1$, and orthogonality among the rows $2,\dots,s$ gives $J+CC^{\top}=sI_r$, where $J=\mathbf 1\mathbf 1^{\top}$; that is, $CC^{\top}=sI_r-J$. Let $W=(J-C)/2$, a $0$-$1$ matrix. Using $JC^{\top}=\mathbf 1(C\mathbf 1)^{\top}=-J$,
\[
W\Big(-\frac2sC^{\top}\Big)=-\frac1s\big(JC^{\top}-CC^{\top}\big)=-\frac1s\big(-J-sI_r+J\big)=I_r,
\]
so $W$ is non-singular with $W^{-1}=-\tfrac2sC^{\top}$, whose entries are $\pm2/s$; hence $\norm{W^{-1}}_{\max}=2/(r+1)$. Now set
\[
A=\begin{pmatrix}I_r&W\\W^{-1}&I_r\end{pmatrix}=\begin{pmatrix}I_r\\W^{-1}\end{pmatrix}\begin{pmatrix}I_r&W\end{pmatrix},\qquad
B=\begin{pmatrix}I_r&W\\0&I_r\end{pmatrix}.
\]
The factorisation shows $\rank A\le r$, and the top-left block $I_r$ shows $\rank A\ge r$. $B$ is a $0$-$1$ matrix, block upper triangular with identity diagonal blocks, so $\det B=1$ and $\rank B=2r$. Finally $A-B$ has the single non-zero block $W^{-1}$, so $\norm{A-B}_{\max}=\norm{W^{-1}}_{\max}=2/(r+1)$. For the consequence, $2/(r+1)\le a/\sqrt r$ as soon as $2\sqrt r/(r+1)\le a$, which holds for all large $r$ since $2\sqrt r/(r+1)\to0$.
\end{proof}

The smallest instance is $r=3$, where the construction gives $W=J_3-I_3$ up to a row permutation, with $\norm{A-B}_{\max}=1/2<1/\sqrt3$. The proposition concerns arbitrary $0$-$1$ matrices only; the band-membership matrices of Section~\ref{sec:size} have additional structure (two bands, one of them extremely thin), and whether that structure admits a rounding statement with a larger tolerance is open.